\documentclass{article}
\PassOptionsToPackage{numbers, compress}{natbib}

\usepackage[main, final]{neurips_2026}

\usepackage[utf8]{inputenc} 
\usepackage[T1]{fontenc}    
\usepackage{hyperref}       
\usepackage{url}            
\usepackage{booktabs}       
\usepackage{amsfonts}       
\usepackage{nicefrac}       
\usepackage{microtype}      
\usepackage[table,xcdraw]{xcolor}        

\usepackage{graphicx} 
\usepackage{amsmath}
\usepackage{algorithm}
\usepackage{algorithmic}

\usepackage{graphicx}
\usepackage{booktabs}
\usepackage{multirow}
\usepackage{colortbl}
\usepackage{arydshln}
\usepackage{pifont}

\definecolor{Gray}{gray}{0.9}
\definecolor{light}{HTML}{E6F0FF}

\newcommand{\cmark}{\ding{51}}%
\newcommand{\xmark}{\ding{55}}%

\newcommand{\equal}{\textsuperscript{*}}        
\newcommand{\corrauth}{\textsuperscript{\textdagger}} 
\newcommand{\sjtu}{\textsuperscript{1}}         
\newcommand{\tongji}{\textsuperscript{2}}       

\usepackage{amssymb}
\usepackage{amsthm}

\newtheorem{proposition}{Proposition}
\usepackage{caption}
\usepackage{subcaption}

\title{From Label Priors to Task Evidence: Long-Video Frame Selection via Bayesian GRPO}

\author{
  \textbf{Shuochen Chang}\sjtu\equal \quad
  \textbf{Bingjie Gao}\sjtu\equal \quad
  \textbf{Qingyang Liu}\sjtu \quad
  \textbf{Qianli Ma}\sjtu \quad
  \textbf{Yibo Miao}\sjtu \\
  \textbf{Haonan Zhao}\sjtu \quad
  \textbf{Zhaohe Liao}\sjtu \quad
  \textbf{Xiaofeng Zhang}\sjtu \quad
  \textbf{Jiangtong Li}\tongji\corrauth \quad
  \textbf{Li Niu}\sjtu\corrauth \\
  \sjtu Shanghai Jiao Tong University \quad \tongji Tongji University \\
  \equal Equal contribution. \quad \corrauth Corresponding authors. \\
  \texttt{\{csc1332741686, whynothaha, ustcnewly\}@sjtu.edu.cn} \\
  \texttt{jiangtongli@tongji.edu.cn}
}

\begin{document}

\maketitle

\begin{abstract}
Multimodal Large Language Models (MLLMs) face significant scalability challenges in long-video understanding due to quadratic attention complexity and context window constraints. Existing retrieval-based approaches typically rely on supervised learning from static keyframe annotations. However, these label priors often fail to align with the actual information required by downstream MLLMs for complex reasoning, creating a misalignment between the training target and inference needs. To bridge this gap, we propose a probabilistic framework that reformulates frame selection as latent-variable inference. We first initialize a lightweight policy using supervised learning to capture general visual saliency. Then, we treat a frozen MLLM as a task environment and use its answer likelihood as dense task evidence to refine the policy via Bayesian Group Relative Policy Optimization (GRPO). Within this formulation, the negative log-likelihood (NLL) over ground-truth answer tokens provides a principled dense reward, while a prior-matching KL term preserves semantic consistency with the elicited label prior. This process transitions the selection principle from static label priors to task evidence. Extensive experiments on Video-MME, MLVU, and LongVideoBench demonstrate that our selector achieves competitive accuracy with MLLM-based scorers while retaining the inference efficiency of lightweight methods. The selector is trained once with a frozen MLLM environment and transfers directly to multiple downstream MLLMs without retraining. Code is available at \url{github.com/Miraclecsc/Frame_Selection_GRPO}.
\end{abstract}

\section{Introduction}
\label{sec:intro}

Recent advancements in Multimodal Large Language Models (MLLMs) have revolutionized video understanding~\cite{bg1,bg2,bg3,bg4,bg5,bg6,bg7,bg8}, enabling sophisticated reasoning over interwoven visual-textual data. However, despite their impressive zero-shot capabilities and instructional alignment, MLLMs face a fundamental struggle when processing long-form videos~\cite{q1,q2,q3,q4}. The quadratic computational complexity inherent to self-attention mechanisms~\cite{attn1}, coupled with strict context window limitations, typically prohibits the direct ingestion of lengthy, uncompressed frame sequences. Consequently, extracting relevant spatiotemporal information from temporally extended content remains a critical bottleneck, hindering the deployment of MLLMs in real-world applications such as movie analysis and long-horizon agent navigation.

To mitigate this scalability issue, existing approaches~\cite{arc1,arc2,chen2026stdec,arc3,re5,re6,huang2026does} generally fall into two dominant paradigms: architectural modification and external frame retrieval. Architectural approaches, such as streaming video LLMs~\cite{arc4,arc5,arc6,arc7,Huang2026RealTimeAR} or dynamic memory compression, attempt to fundamentally modify the MLLM's internal token-processing mechanism to accommodate infinite or highly compressed contexts. While theoretically promising, these methods often incur prohibitive training costs and suffer from catastrophic forgetting, degrading the model's general-purpose capabilities. Alternatively, retrieval-based methods~\cite{re1,re2,re3,re4} bypass architectural changes by employing an external module to select a concise subset of keyframes. However, this paradigm suffers from a severe efficiency-accuracy trade-off. Lightweight selectors (\emph{e.g.}, CLIP-based similarity matchers)~\cite{aks,liu2026palm,bolt,f2c,frameoracle,frame-voyager,huang2025adaptive} rely on shallow semantic alignment and lack the deep causal reasoning required for complex, multi-hop temporal queries. Conversely, utilizing a powerful MLLM as a dense frame scorer~\cite{air,keyframes,gens,viarl} bridges the reasoning gap but introduces unacceptable computational overhead and latency during inference.

In this work, we propose a lightweight, plug-and-play selection module that achieves high reasoning accuracy without imposing heavy inference latency. We identify a fundamental theoretical misalignment in current retrieval-based methods: \textit{a static label prior is not strictly identical to task evidence}. Standard supervised selectors trained on fixed human annotations learn the former (typically visually salient or summary frames), but downstream MLLMs require the latter (specific visual evidence necessary to complete a logical reasoning chain). To bridge this semantic gap, we depart from a ranking paradigm and reformulate frame selection as a probabilistic latent-variable modeling problem.

Specifically, we introduce a rigorous two-stage probabilistic framework that transitions from human-derived label priors to MLLM-driven task evidence. First, we construct a lightweight selector that learns a \textit{Label Prior} from keyframe annotations via pairwise ranking. This initializes the policy within a semantically plausible region, drastically reducing the combinatorial search space. Second, we perform \textit{Posterior Inference} to dynamically align this prior with the target MLLM. By treating the frozen MLLM as the environment and maximizing the Evidence Lower Bound (ELBO), we naturally utilize the MLLM's Negative Log-Likelihood (NLL) as dense task evidence, while simultaneously leveraging the elicited prior for principled KL-regularization. We solve this optimization using a novel Bayesian Group Relative Policy Optimization (GRPO) algorithm, which effectively refines the policy into a task-adaptive posterior that autonomously identifies reasoning-critical frames without reward hacking.

Our main contributions are summarized as follows:
\begin{itemize}
    \item We propose a plug-and-play, lightweight frame selector that decouples selection from inference, enabling highly efficient long-video understanding without modifying the target MLLM's architecture.
    \item We introduce a probabilistic framework that reformulates frame selection as Bayesian inference, theoretically justifying the sequential transition from supervised Prior Elicitation to MLLM-driven Posterior Inference.
    \item We derive a novel Bayesian GRPO algorithm. By leveraging the target MLLM's answer likelihood as dense task evidence and the label prior as KL-regularization, our method robustly refines the selector while mitigating reward hacking.
    \item We empirically demonstrate that our selector acts as a universal plugin across multiple downstream MLLMs, achieving a superior accuracy-efficiency trade-off compared to both training-free lightweight methods and computationally heavy MLLM-based scorers.
\end{itemize}
\section{Related Works}

\subsection{MLLM for Long Video Understanding}
Recent advancements in Multimodal Large Language Models (MLLMs) have significantly improved video understanding~\cite{rela1,shen2026decoding,rela2,rela3,zhang2026heterogeneous,li2026sage}, yet processing long-form videos remains a computational bottleneck. Existing approaches largely focus on architectural modifications to handle extended context windows. For instance, StreamingVLM~\cite{arc1} introduces a streaming-aware mechanism to enable infinite-length video processing without retraining, while Video-XL~\cite{videoxl,videoxlp,yu2025core3d,videoxl2} utilizes latent summarization tokens to compress hour-scale visual inputs efficiently. Other works like TimeSuite~\cite{timesuite,wang2026solving} propose grounded tuning strategies to adapt short-form MLLMs for long contexts. These architecture-centric methods often require substantial computational resources or sacrifice fine-grained temporal reasoning capabilities, highlighting the need for more efficient, selection modules.

\subsection{Bayesian Methods in Large Models}
Bayesian frameworks~\cite{bayes3,bayes4,bayes5,fu2025learning,qin2026sstoken,qin2026multi} have seen a resurgence for quantifying uncertainty and integrating priors in large models. Recent studies such as Bayesian Evaluation of LLM Behavior~\cite{bayes1} and Bayesian Mixture of Experts~\cite{bayes2} demonstrate that probabilistic modeling can effectively calibrate model outputs and mitigate hallucinations. Recent works~\cite{bayes6,bayes7} formulates LLMs as prior action distributions, showing that Bayesian inference can significantly accelerate exploration in reinforcement learning. This line of research provides strong theoretical grounding for treating supervised signals as priors and downstream task feedback as evidence, a core tenet of our proposed framework.

\subsection{Policy Optimization and Reinforcement Learning for MLLMs}
Reinforcement learning has become a standard paradigm for aligning MLLMs~\cite{vrl1,vrl2,vrl3,vrl4,vrl5,huang2026deferred,shen2026fine,zhang2026xhc} with complex reasoning tasks. General-purpose reasoning has been enhanced by algorithms like Hint-GRPO~\cite{vrl6,li2026toward,zhang2026nitp}, which applies Group Relative Policy Optimization to boost mathematical and multimodal reasoning capabilities. In the specific domain of frame selection, recent works have explored various decision-making formulations. Focus~\cite{focus} models keyframe selection as a combinatorial multi-armed bandit problem to balance exploration and exploitation. Similarly, FrameMind~\cite{vrl7,yu2026elsa3d} employs trajectory-level RL to dynamically request visual information during reasoning. Unlike these methods which often rely on sparse rewards or heuristic bandits, our approach leverages a dense NLL-based reward signal within a rigorous Bayesian inference framework to achieve task-adaptive frame selection.

\section{Methodology}
\label{sec:method}

\begin{figure*}[t]
  \centering
  \includegraphics[width=\linewidth]{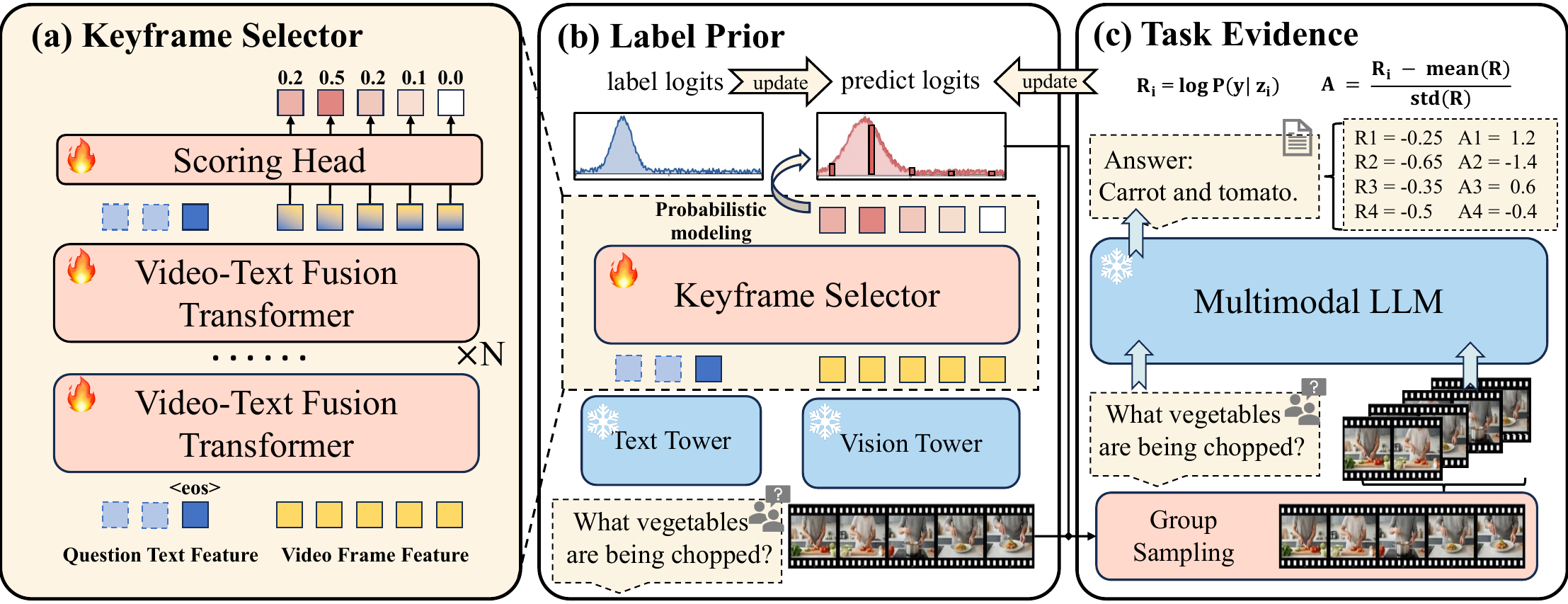}
  \caption{
    \textbf{The overall framework of our methods.} 
    \textbf{(a) Keyframe Selector.} The lightweight selector ($f_\theta$) comprises stacked Video-Text Fusion Transformers. It modulates frame features with query semantics to output importance logits, then parameterizing a Plackett-Luce distribution.
    \textbf{(b) Stage I: Label Prior Elicitation.} We utilize keyframe label annotations to construct a \textit{Label Prior}. The selector is trained via pairwise ranking loss to align its predicted logits with the label logits, establishing a robust initialization.
    \textbf{(c) Stage II:Task-Evidence Posterior Inference.} We freeze the MLLM and treat it as a frozen task environment. The selector performs Group Sampling to propose discrete frame subsets ($\mathbf{z}$). The MLLM evaluates these subsets to provide dense NLL rewards ($R$) and advantages ($A$). The policy is refined via \textbf{Bayesian GRPO} to maximize task evidence while maintaining KL-regularization towards the label prior.
  }
  \label{fig:pipeline}
\end{figure*}

We formulate long-video frame selection as a probabilistic latent-variable modeling problem, as shown in Figure~\ref{fig:pipeline}. We use $\mathcal{V} \in \mathbb{R}^{N \times D}$ to denote a sequence of $N$ candidate frames uniformly sampled from the raw video, $\mathcal{Q}$ to denote a textual query, and $\mathbf{y}$ to denote the ground-truth answer. Here, $N$ is reserved for the number of candidate frames, while $L$ denotes the number of Transformer blocks in the selector.

Our objective is to learn a parameterized policy $\pi_\theta(\mathbf{z} \mid \mathcal{V}, \mathcal{Q})$ that selects a discrete subset $\mathbf{z} \subset \{1, \dots, N\}$ with cardinality $|\mathbf{z}| = K$, maximizing the marginal likelihood of the correct answer under a frozen Multimodal Large Language Model.

\subsection{Variational Policy Generator}
\label{sec:model}

The true posterior distribution $P(\mathbf{z} \mid \mathbf{y}, \mathcal{V}, \mathcal{Q})$ is intractable due to the combinatorial nature of the discrete space $\binom{N}{K}$. To address this, we employ variational inference by introducing a variational family $q_\theta(\mathbf{z} \mid \mathcal{V}, \mathcal{Q})$ to approximate the posterior.

We implement the variational distribution using a lightweight selector network $f_\theta$. The network projects the visual and textual inputs into a sequence of unnormalized importance scores (logits) $\mathbf{s} \in \mathbb{R}^N$:
\begin{equation}
\label{eq:logits}
\mathbf{s} = f_\theta(\mathcal{V}, \mathcal{Q}).
\end{equation}
To define a valid probability measure over ordered discrete subsets without replacement, we parameterize $q_\theta$ using the Plackett-Luce model. The probability of sampling an ordered sequence $\mathbf{z} = (z_1, \dots, z_K)$ is factorized as:
\begin{equation}
\label{eq:pl_policy}
q_\theta(\mathbf{z} \mid \mathbf{s}) = \prod_{k=1}^{K} \frac{\exp(s_{z_k})}{\sum_{j \in \mathcal{U}_k} \exp(s_j)},
\end{equation}
where $\mathcal{U}_k = \{1, \dots, N\} \setminus \{z_1, \dots, z_{k-1}\}$ represents the set of available frames at step $k$. This formulation ensures a normalized probability distribution over the combinatorial action space, enabling consistent gradient estimation.

\subsection{Stage I: Prior Elicitation via Ranking}
\label{sec:prior}

In Bayesian inference, the quality of the prior distribution $p(\mathbf{z})$ critically dictates the efficiency of posterior convergence. A non-informative uniform prior over the vast search space $\binom{N}{K}$ results in severe sparsity issues. We therefore interpret the first training stage as \textbf{Prior Elicitation}, leveraging human-annotated keyframes to establish an informative prior $\pi_{\mathrm{prior}}(\mathbf{z})$.

We construct a pairwise preference dataset $\mathcal{D} = \{ (i,j) \mid y_i > y_j \}$ derived from human annotations. We model the local relative saliency using the Bradley-Terry-Luce model:
\begin{equation}
\label{eq:btl}
P(i \succ j \mid \mathbf{s}) = \frac{\exp(s_i)}{\exp(s_i) + \exp(s_j)} = \sigma(s_i - s_j).
\end{equation}
The parameters are estimated by minimizing the negative log-likelihood of the observed preferences:
\begin{equation}
\label{eq:prior_loss}
\mathcal{L}_{\mathrm{prior}}(\theta) = -\mathbb{E}_{(i,j) \sim \mathcal{D}} \left[ \log \sigma(s_i - s_j) \right].
\end{equation}
The resulting policy $\pi_{\mathrm{prior}}$ restricts the support of the latent distribution to a semantically plausible region, effectively initializing the variational approximation near the modes of human visual perception.

\subsection{Stage II: Task-Evidence Posterior Inference}
\label{sec:posterior}

While $\pi_{\mathrm{prior}}$ captures general saliency, it remains agnostic to the specific reasoning dynamics of the MLLM. To align the selection with task requirements, we perform posterior inference using the frozen MLLM as the likelihood function $P_{\mathrm{MLLM}}(\mathbf{y} \mid \mathbf{z}, \mathcal{Q})$.

\subsubsection{Derivation of the Evidence Lower Bound}

We aim to maximize the log-marginal likelihood of the ground-truth answer $\mathbf{y}$. By marginalizing over the latent variable $\mathbf{z}$, we have:
\begin{equation}
\label{eq:marginal}
\log P(\mathbf{y} \mid \mathcal{V}, \mathcal{Q}) = \log \sum_{\mathbf{z}} P_{\mathrm{MLLM}}(\mathbf{y} \mid \mathbf{z}) \pi_{\mathrm{prior}}(\mathbf{z}).
\end{equation}
Direct maximization of Eq.~\ref{eq:marginal} is intractable. We resort to maximizing its lower-bound by introducing the variational approximation $q_\theta(\mathbf{z})$ and applying Jensen's Inequality:
\begin{equation}
\label{eq:elbo_deriv_step1}
\begin{aligned}
\log P(\mathbf{y}) &= \log \mathbb{E}_{\mathbf{z} \sim q_\theta} \left[ \frac{P_{\mathrm{MLLM}}(\mathbf{y} \mid \mathbf{z}) \pi_{\mathrm{prior}}(\mathbf{z})}{q_\theta(\mathbf{z})} \right] \\
&\ge \mathbb{E}_{\mathbf{z} \sim q_\theta} \left[ \log \left( \frac{P_{\mathrm{MLLM}}(\mathbf{y} \mid \mathbf{z}) \pi_{\mathrm{prior}}(\mathbf{z})}{q_\theta(\mathbf{z})} \right) \right] \\
&= \mathcal{L}_{\mathrm{ELBO}}(\theta).
\end{aligned}
\end{equation}
Expanding the logarithmic term reveals the structural components of our objective:
\begin{equation}
\label{eq:elbo_expanded}
\mathcal{L}_{\mathrm{ELBO}}(\theta) = \mathbb{E}_{\mathbf{z} \sim q_\theta} \left[ \log P_{\mathrm{MLLM}}(\mathbf{y} \mid \mathbf{z}) \right] - D_{\mathrm{KL}}(q_\theta \| \pi_{\mathrm{prior}}).
\end{equation}
The first term represents the \textit{Task Evidence} (Likelihood), ensuring the selected frames support the correct answer. The second term serves as \textit{Semantic Regularization}, explicitly constraining the policy towards the human-derived prior $\pi_{\mathrm{prior}}$ to prevent reward hacking.

The above derivation provides a principled foundation for our reward design within the latent-variable formulation in Eq.~\ref{eq:marginal}. Under this formulation, the answer likelihood term naturally induces the token-level negative log-likelihood (NLL) over the ground-truth answer as a dense task-evidence signal. Furthermore, the KL term explicitly regularizes the policy towards the \textit{Prior} (human semantics) rather than an arbitrary previous policy, ensuring semantic consistency.

\subsubsection{The Tempered Posterior Objective}

To control the strength of the regularization and improve optimization stability, we introduce a temperature coefficient $\beta$. The final maximization objective is:
\begin{equation}
\label{eq:final_obj}
\mathcal{J}(\theta)
=
\mathbb{E}_{\mathbf{z} \sim q_\theta} [R(\mathbf{z})]
-
\beta D_{\mathrm{KL}}(q_\theta \| \pi_{\mathrm{prior}}),
\end{equation}
where $R(\mathbf{z}) \triangleq \log P_{\mathrm{MLLM}}(\mathbf{y} \mid \mathbf{z})$. We show that the optimal policy under this objective takes the form of a tempered posterior.

\begin{proposition}
The variational distribution $q^*(\mathbf{z})$ that maximizes $\mathcal{J}(\theta)$ is given by
\begin{equation}
\label{eq:tempered_posterior}
q^*(\mathbf{z})
=
\frac{1}{Z}
\pi_{\mathrm{prior}}(\mathbf{z})
\exp\left(\frac{R(\mathbf{z})}{\beta}\right),
\end{equation}
where $Z$ is the normalizing constant.
\end{proposition}

\begin{proof}
Consider the KL divergence from $q_\theta$ to the tempered posterior $q^*$:
\begin{align}
D_{\mathrm{KL}}(q_\theta \| q^*)
&=
\mathbb{E}_{q_\theta}
\left[
\log q_\theta(\mathbf{z})
-
\log q^*(\mathbf{z})
\right] \notag \\
&=
\mathbb{E}_{q_\theta}
\left[
\log q_\theta(\mathbf{z})
-
\log \pi_{\mathrm{prior}}(\mathbf{z})
-
\frac{1}{\beta} R(\mathbf{z})
+
\log Z
\right] \notag \\
&=
D_{\mathrm{KL}}(q_\theta \| \pi_{\mathrm{prior}})
-
\frac{1}{\beta}
\mathbb{E}_{q_\theta}[R(\mathbf{z})]
+
\log Z \notag \\
&=
-\frac{1}{\beta}
\left(
\mathbb{E}_{q_\theta}[R(\mathbf{z})]
-
\beta D_{\mathrm{KL}}(q_\theta \| \pi_{\mathrm{prior}})
\right)
+
\log Z \notag \\
&=
-\frac{1}{\beta}\mathcal{J}(\theta)
+
\log Z .
\end{align}
Since $Z$ is independent of $q_\theta$, maximizing $\mathcal{J}(\theta)$ is equivalent to minimizing $D_{\mathrm{KL}}(q_\theta \| q^*)$. Therefore, the optimal variational distribution is the tempered posterior in Eq.~\ref{eq:tempered_posterior}.
\end{proof}

\subsection{Bayesian Group Relative Policy Optimization}

Optimization of Eq.~\ref{eq:final_obj} requires estimating gradients over discrete expectations. We derive a Bayesian variant of Group Relative Policy Optimization (GRPO). Unlike standard RL which constrains updates against a previous policy $\pi_{\mathrm{old}}$, our derivation requires specific regularization towards $\pi_{\mathrm{prior}}$.

We utilize the log-derivative trick to estimate the gradient of the expected reward:
\begin{equation}
\nabla_\theta \mathbb{E}_{q_\theta}[R(\mathbf{z})] = \mathbb{E}_{q_\theta} [R(\mathbf{z}) \nabla_\theta \log q_\theta(\mathbf{z})].
\end{equation}
To reduce variance, we employ group-based importance sampling. For each query, we sample a group of $G$ trajectories $\{\mathbf{z}_i\}_{i=1}^G$ from the behavior policy $\pi_{\mathrm{old}}$. The normalized advantage $\hat{A}_i$ is computed relative to the group mean:
\begin{equation}
\hat{A}_i = \frac{R(\mathbf{z}_i) - \frac{1}{G}\sum_{j} R(\mathbf{z}_j)}{\sigma_R + \epsilon}.
\end{equation}
In practice, $R(\mathbf{z})$ is computed only over the ground-truth answer tokens under a unified short-form answer format. The group-relative normalization in Eq.~\ref{eq:grpo_grad} makes the policy update depend on the relative ranking of sampled frame subsets for the same query. Consequently, the optimization is robust to moderate variations in answer length and output format. The final gradient estimator is given by:
\begin{equation}
\label{eq:grpo_grad}
\nabla_\theta \mathcal{J} \approx \frac{1}{G} \sum_{i=1}^{G} \bigg( \hat{A}_i \nabla_\theta \log q_\theta(\mathbf{z}_i) - \beta \nabla_\theta D_{\mathrm{KL}}(q_\theta(\mathbf{z}_i) \| \pi_{\mathrm{prior}}(\mathbf{z}_i)) \bigg).
\end{equation}
In practice, we implement this using the clipped surrogate objective to ensure monotonic improvement:
\begin{equation}
\label{eq:grpo_loss}
\mathcal{L}_{\mathrm{GRPO}}(\theta) = -\frac{1}{G} \sum_{i=1}^{G} \min\left(\rho_i \hat{A}_i, \text{clip}(\rho_i, 1-\epsilon, 1+\epsilon)\hat{A}_i\right) + \beta D_{\mathrm{KL}}(q_\theta \| \pi_{\mathrm{prior}}),
\end{equation}
where $\rho_i = q_\theta(\mathbf{z}_i) / \pi_{\mathrm{old}}(\mathbf{z}_i)$ is the importance sampling ratio. This objective completes the Bayesian loop: the advantage term $\hat{A}_i$ drives the policy towards the likelihood (Task Evidence), while the KL term anchors it to the learned prior, strictly adhering to the ELBO formulation. The complete training procedure, illustrating the transition from Prior Elicitation to Posterior Inference, is summarized in Algorithm~\ref{alg:bayes_grpo}.

\section{Experiment}
\label{sec:exp}

\begin{table*}[t]
\small
\centering
\caption{Comparison of frame selection methods on \textbf{Qwen2.5-VL-7B}. Methods are categorized by type and grouped by frame count. \textbf{Bold}/\underline{underline} denote the best/second-best performance.}
\label{tab:qwen_results}
\resizebox{\textwidth}{!}{%
\begin{tabular}{lcccccccc}
\toprule
\multirow{2}{*}{\textbf{Method}} & \multirow{2}{*}{\textbf{\# Frames}} & \multicolumn{4}{c}{\textbf{Video-MME}} & \multirow{2}{*}{\textbf{LVB}} & \multicolumn{2}{c}{\textbf{MLVU}} \\ 
\cmidrule(lr){3-6} \cmidrule(lr){8-9}
 & & Short & Medium & Long & Overall & & Needle-QA & Avg \\ 
\textit{Avg. Video Duration} & & \textit{1.3min} & \textit{9min} & \textit{41min} & \textit{17min} & \textit{12min} & \multicolumn{2}{c}{\textit{12min}} \\
\midrule  

\rowcolor{Gray} \multicolumn{9}{l}{\textit{\textbf{Uniform Sampling Baseline}}} \\
Uniform & 8 & 61.7 & 50.6 & 45.7 & 52.6 & 52.6 & 62.8 & 56.1 \\
Uniform & 32 & 71.4 & 60.8 & 49.2 & 60.5 & 56.3 & 72.4 & 62.2 \\ 
\midrule

\rowcolor{Gray} \multicolumn{9}{l}{\textit{\textbf{Lightweight Scoring Methods}}} \\
AKS & 8 & 62.1 & 51.4 & 48.8 & 54.1 & 52.0 & - & 53.7 \\
BOLT & 8 & 66.0 & 54.6 & 50.4 & 57.0 & 55.6 & - & 59.0 \\
Q-Frame & 8 & \underline{71.9} & 58.2 & 50.7 & 60.3 & 56.6 & - & 58.2 \\
F2C & 8 & \textbf{72.4} & \textbf{60.7} & \underline{51.2} & \underline{61.4} & \textbf{58.9} & - & \underline{63.1} \\
\rowcolor{light}
\textbf{Ours} & 8 & 71.2 & \underline{60.4} & \textbf{52.9} & \textbf{61.5} & \underline{58.2} & \textbf{82.0} & \textbf{66.1} \\
\cdashline{1-9}
MDP3 & 32 & - & - & - & 63.8 & 60.0 & - & 66.2 \\
AKS & 32 & 66.4 & 56.9 & 48.6 & 57.3 & 58.3 & - & 63.1 \\
BOLT & 32 & 74.3 & 64.2 & 53.8 & 64.1 & 58.6 & - & 66.3 \\
Q-Frame & 32 & 73.0 & 61.7 & 53.1 & 62.6 & 58.7 & - & - \\
F2C & 32 & \underline{73.8} & \textbf{66.3} & \textbf{56.6} & \textbf{65.6} & 60.8 & - & 66.8 \\
\rowcolor{light}
\textbf{Ours} & 32 & \textbf{74.1} & \underline{65.8} & \underline{55.7} & \underline{65.2} & \underline{61.1} & \textbf{82.5} & \textbf{70.3} \\ 
\midrule

\rowcolor{Gray} \multicolumn{9}{l}{\textit{\textbf{MLLM-based Scoring Methods}}} \\
ViaRL & 8 & 65.1 & 56.1 & 50.8 & 57.3 & - & 73.5 & 58.2 \\
Kframes & 8 & 68.9 & 55.3 & 47.9 & 57.4 & 57.7 & \underline{77.5} & 60.4 \\
\cdashline{1-9}
Kframes & 32 & \textbf{74.1} & 61.4 & 51.7 & 62.1 & 60.5 & \underline{79.4} & 65.9 \\
A.I.R. & 32 & - & - & - & 65.0 & \textbf{61.4} & - & \underline{67.5} \\ 
GenS & 54 & - & - & - & - & 58.7 & - & 64.8 \\
\bottomrule
\end{tabular}%
}
\end{table*}

\begin{table*}[t]
\small
\centering
\caption{Comparison of frame selection methods on \textbf{InternVL-3-8B} and \textbf{LLaVA-OneVision-7B}. Methods are categorized by type, grouped by model, and then by frame count. \textbf{Bold}/\underline{underline} denote the best/second-best performance.}
\label{tab:merged_results}
\resizebox{\textwidth}{!}{%
\begin{tabular}{lcccccccc}
\toprule
\multirow{2}{*}{\textbf{Method}} & \multirow{2}{*}{\textbf{\# Frames}} & \multicolumn{4}{c}{\textbf{Video-MME}} & \multirow{2}{*}{\textbf{LVB}} & \multicolumn{2}{c}{\textbf{MLVU}} \\ 
\cmidrule(lr){3-6} \cmidrule(lr){8-9}
 & & Short & Medium & Long & Overall & & Needle-QA & Avg \\ 
\textit{Avg. Video Duration} & & \textit{1.3min} & \textit{9min} & \textit{41min} & \textit{17min} & \textit{12min} & \multicolumn{2}{c}{\textit{12min}} \\
\midrule  

\rowcolor{Gray} \multicolumn{9}{l}{\textit{\textbf{Uniform Sampling Baseline}}} \\
\multicolumn{9}{l}{\textbf{InternVL-3-8B}} \\
Uniform & 8 & 68.7 & 58.3 & 48.8 & 58.6 & 53.4 & 77.2 & 64.4 \\
Uniform & 32 & 74.2 & 62.8 & 51.7 & 62.9 & 57.9 & 83.7 & 70.3 \\ 
\cdashline{1-9}
\multicolumn{9}{l}{\textbf{LLaVA-OneVision-7B}} \\
Uniform & 8 & 62.6 & 51.6 & 44.3 & 52.8 & 52.7 & 63.7 & 60.2 \\
Uniform & 32 & 69.3 & 54.7 & 46.2 & 56.7 & 55.2 & 76.6 & 64.7 \\ 
\midrule
\rowcolor{Gray} \multicolumn{9}{l}{\textit{\textbf{Lightweight Scoring Methods}}} \\
\multicolumn{9}{l}{\textbf{InternVL-3-8B}} \\
\rowcolor{light}
\textbf{Ours} & 8 & \textbf{74.7} & \textbf{60.9} & \textbf{51.6} & \textbf{62.4} & \textbf{58.7} & \textbf{85.1} & \textbf{69.0} \\
\cdashline{1-9}
MDP3 & 32 & - & - & - & \underline{66.8} & 60.9 & - & 74.0 \\
\rowcolor{light}
\textbf{Ours} & 32 & \textbf{76.2} & \textbf{65.9} & \textbf{58.4} & \underline{66.8} & \underline{62.1} & \textbf{85.4} & \underline{74.3} \\ 
\cdashline{1-9}
\multicolumn{9}{l}{\textbf{LLaVA-OneVision-7B}} \\
Frame-Voyager & 8 & 67.3 & \underline{56.3} & \underline{48.9} & 57.5 & - & - & 65.6 \\
FrameOracle & 8 & - & - & - & 57.5 & 56.0 & - & 62.9 \\
MDP3 & 8 & \textbf{71.8} & \textbf{57.6} & \textbf{49.6} & \textbf{59.6} & \textbf{59.0} & - & \textbf{69.8} \\
BOLT & 8 & 66.8 & 54.2 & 47.3 & 56.1 & 55.6 & - & 63.4 \\
Nar-KFC & 8 & - & - & - & \underline{57.8} & 56.5 & - & 66.2 \\
\rowcolor{light}
\textbf{Ours} & 8 & \underline{68.3} & \underline{56.3} & 48.2 & 57.6 & \underline{57.8} & 74.3 & \underline{67.8} \\
\cdashline{1-9}
Focus & 32 & - & - & - & 58.3 & 60.7 & - & - \\
AKS & 32 & - & - & - & 58.4 & 59.3 & - & - \\
MDP3 & 32 & - & - & - & 60.5 & 59.0 & - & 68.3 \\
BOLT & 32 & \underline{70.1} & \underline{60.0} & \underline{49.6} & 59.9 & 59.6 & - & 66.8 \\
\rowcolor{light}
\textbf{Ours} & 32 & \textbf{70.2} & \textbf{61.4} & \textbf{50.7} & \underline{60.8} & \textbf{60.8} & \textbf{76.1} & \underline{69.2} \\ 
\midrule

\rowcolor{Gray} \multicolumn{9}{l}{\textit{\textbf{MLLM-based Scoring Methods}}} \\
\multicolumn{9}{l}{\textbf{InternVL-3-8B}} \\
A.I.R. & 32 & - & - & - & \textbf{68.2} & \textbf{62.8} & - & \textbf{74.5} \\ 
\cdashline{1-9}
\multicolumn{9}{l}{\textbf{LLaVA-OneVision-7B}} \\
A.I.R. & 32 & - & - & - & \textbf{61.4} & \underline{60.7} & - & \textbf{69.3} \\ 
\bottomrule
\end{tabular}%
}
\end{table*}

\subsection{Experimental Setup}
\label{sec:setup}

\paragraph{Models and Baselines.}
We compare our approach against two distinct categories of frame selection strategies, lightweight scoring and MLLM-based methods. The first category consists of lightweight scoring methods. This includes methods utilizing efficient vision-language models, such as SigLIP2~\cite{siglip2}, to measure semantic alignment. Specific approaches include AKS~\cite{aks}, which segments videos based on adaptive similarity; BOLT~\cite{bolt}, which computes importance scores followed by inverse transform sampling; and Q-Frame~\cite{qframe}, which employs dynamic resolution based on cross-modal retrieval. F2C~\cite{f2c} utilizes CLIP-based anchor keys to adjust clip lengths during selection. Optimization-based methods like MDP3~\cite{mdp3} and Nar-KFC~\cite{narkfc} formulate selection as subset or subgraph problems using mathematical estimation. Additionally, Focus~\cite{focus} employs multi-armed bandits to balance exploration and exploitation, while Frame-Voyager~\cite{frame-voyager} and FrameOracle~\cite{frameoracle} use pairwise ranking or joint importance prediction to identify keyframe sets. The second category, MLLM-based methods, leverages heavy multimodal models for direct scoring. Examples include ViaRL~\cite{viarl} and Kframes~\cite{keyframes}, which apply reinforcement learning for temporal grounding; A.I.R.~\cite{air}, which uses a two-stage iterative selection process; and GenS~\cite{gens}, which treats frame scoring as a generative sampling task. Results are reported over multiple runs. \textbf{Efficiency and Training Cost Analysis} are provided in Appendix~\ref{app:exp_details}.

\paragraph{Benchmarks and Implementation Details.} We evaluate our method on three widely adopted long-form video benchmarks, covering diverse domains and reasoning granularities. Details about benchmarks are provided in Appendix~\ref{app:exp_details}. We utilize SigLIP2-SO400M-Patch16-512~\cite{siglip2} as the shared visual-textual backbone for feature pre-extraction. More details are provided in Appendix~\ref{app:exp_details}.

\subsection{Results on Long-form Video Benchmarks}
\label{sec:main_results}

We present the quantitative evaluation results on Video-MME, LongVideoBench, and MLVU under different frame budgets in Tables~\ref{tab:qwen_results}, \ref{tab:merged_results}, and \ref{tab:transferability}. Notably, our approach maintains competitive performance with all frame budget, demonstrating the efficacy of our keyframe selection strategy.

\paragraph{Comparison with State-of-the-Art methods.}
As evidenced in Tables~\ref{tab:qwen_results} and~\ref{tab:merged_results}, our method significantly outperforms existing lightweight selectors while achieving parity with computationally intensive MLLM-based approaches. Our model with $K=8$ frames surpasses the strong uniform baseline by a substantial margin of 8.9\% on Video-MME and 10.0\% on MLVU and outperforming lightweight methods such as BOLT and Q-Frame on Qwen2.5-VL(Table~\ref{tab:qwen_results}). Our method even achieves an overall accuracy of 66.8\% with $K=32$ which effectively matches the performance of the MLLM-based scorer A.I.R. on InternVL-3(Table~\ref{tab:merged_results}). Furthermore, compared to mathematical iterative methods like Nar-KFC and MDP3 that demand extensive test-time computation, our module offers a direct single-pass forward inference.


\paragraph{Generalization to State-of-the-Art MLLMs.}
We extend the evaluation to SOTA MLLMs: Qwen3-VL-8B-Instruct (no-thinking) and InternVL-3.5-8B (thinking) in Table~\ref{tab:transferability} to assess the generalization. On Qwen3-VL ($K=32$), our method improves the overall Video-MME score from 63.5\% to 65.3\% and the MLVU average from 65.3\% to 72.7\%. We observe that the relative margin over uniform sampling is slightly compressed compared to earlier architectures (e.g., Qwen2.5-VL), which can be attributed to the enhanced reasoning robustness of SOTA MLLMs from suboptimal visual contexts or partial evidence.

\begin{table*}[t]
\small
\centering
\caption{Transferability analysis on state-of-the-art MLLMs (\textbf{Qwen3-VL-8B} and \textbf{InternVL-3.5-8B}). We compare the original model performance (Uniform) with our selector. \textbf{Bold} denotes the best performance.}
\label{tab:transferability}
\resizebox{\textwidth}{!}{%
\begin{tabular}{lcccccccc}
\toprule
\multirow{2}{*}{\textbf{Method}} & \multirow{2}{*}{\textbf{\# Frames}} & \multicolumn{4}{c}{\textbf{Video-MME}} & \multirow{2}{*}{\textbf{LVB}} & \multicolumn{2}{c}{\textbf{MLVU}} \\
\cmidrule(lr){3-6} \cmidrule(lr){8-9}
 & & Short & Medium & Long & Overall & & Needle-QA & Avg \\
\textit{Avg. Video Duration} & & \textit{1.3min} & \textit{9min} & \textit{41min} & \textit{17min} & \textit{12min} & \multicolumn{2}{c}{\textit{12min}} \\
\midrule

\rowcolor{Gray} \multicolumn{9}{l}{\textbf{Qwen3-VL-8B}} \\
Uniform & 8 & 64.8 & 51.8 & 47.7 & 54.7 & 53.3 & 60.4 & 56.9 \\
\textbf{Ours} & 8 & \textbf{67.8} & \textbf{56.2} & \textbf{50.8} & \textbf{58.3} & \textbf{58.2} & \textbf{79.8} & \textbf{66.7} \\
\cdashline{1-9}
Uniform & 32 & 76.1 & 61.6 & 52.8 & 63.5 & 58.9 & 71.6 & 65.3 \\
\textbf{Ours} & 32 & \textbf{77.6} & \textbf{63.6} & \textbf{54.8} & \textbf{65.3} & \textbf{62.7} & \textbf{82.3} & \textbf{72.7} \\
\midrule

\rowcolor{Gray} \multicolumn{9}{l}{\textbf{InternVL-3.5-8B}} \\
Uniform & 8 & 64.7 & 55.9 & 46.9 & 55.8 & 56.7 & 75.5 & 62.7 \\
\textbf{Ours} & 8 & \textbf{70.6} & \textbf{56.2} & \textbf{48.2} & \textbf{58.3} & \textbf{59.6} & \textbf{86.8} & \textbf{67.6} \\
\cdashline{1-9}
Uniform & 32 & 72.9 & 60.0 & 50.3 & 61.1 & 60.1 & 83.1 & 68.5 \\
\textbf{Ours} & 32 & \textbf{76.0} & \textbf{63.2} & \textbf{52.9} & \textbf{64.1} & \textbf{63.4} & \textbf{85.4} & \textbf{71.4} \\
\bottomrule
\end{tabular}%
}
\end{table*}

\begin{table}[t]
\centering
\small
\caption{\textbf{Cross-environment Stage II training.}
Both selectors are evaluated on InternVL-3-8B with $K=8$ and $K=32$.
The selector trained with Qwen2.5-VL transfers directly to InternVL-3 without additional posterior training and achieves comparable performance to the selector trained with InternVL-3 itself.
\textbf{Bold} denotes the best performance.}
\label{tab:cross_env_training}
\setlength{\tabcolsep}{5.5pt}
\renewcommand{\arraystretch}{1.12}
\begin{tabular}{lcccccc}
\toprule
\multirow{2}{*}{\textbf{Stage II Env.}} 
& \multicolumn{2}{c}{\textbf{Video-MME}} 
& \multicolumn{2}{c}{\textbf{LVB}} 
& \multicolumn{2}{c}{\textbf{MLVU}} \\
\cmidrule(lr){2-3} \cmidrule(lr){4-5} \cmidrule(lr){6-7}
& $K=8$ & $K=32$ & $K=8$ & $K=32$ & $K=8$ & $K=32$ \\
\midrule

\rowcolor{Gray}
\multicolumn{7}{l}{\textbf{Evaluation: InternVL-3-8B}} \\
Qwen2.5-VL-7B 
& \textbf{62.4} & \textbf{66.8} 
& 58.7 & \textbf{62.1} 
& 69.0 & \textbf{74.3} \\
InternVL-3-8B 
& \textbf{62.4} & 66.5 
& \textbf{59.0} & 61.9 
& \textbf{69.2} & 74.2 \\

\bottomrule
\end{tabular}
\end{table}

\begin{table*}[t]
\small
\centering
\caption{\textbf{Ablation study on training strategies.} All results are evaluated using the \textbf{Qwen2.5-VL-7B} model. We investigate the impact of the Label Prior (Stage I) and the Bayesian GRPO update (Stage II) across different frame budgets ($K=8, 32$). Reward Source indicates the reward signal used during GRPO.}
\label{tab:ablation_component}
\resizebox{\textwidth}{!}{%
\begin{tabular}{l ccc | ccc ccc}
\toprule
\multirow{2}{*}{\textbf{Method}} & \multicolumn{3}{c}{\textbf{Training Components}} & \multicolumn{3}{c}{\textbf{K = 8}} & \multicolumn{3}{c}{\textbf{K = 32}} \\
\cmidrule(lr){2-4} \cmidrule(lr){5-7} \cmidrule(lr){8-10}
 & \textbf{Stage I} & \textbf{Stage II} & \textbf{Reward} & \textbf{V-MME} & \textbf{LVB} & \textbf{MLVU} & \textbf{V-MME} & \textbf{LVB} & \textbf{MLVU} \\
\midrule
SigLIP2 & \xmark & \xmark & - & 54.7 & 56.6 & 61.9 & 61.1 & 59.7 & 66.3 \\
\midrule
Label Prior & \cmark & \xmark & - & 58.5 & 57.4 & 65.2 & 63.6 & 60.4 & 68.2 \\
+ Label Reward & \cmark & \cmark & label & 58.3 & 57.6 & 64.9 & 63.8 & 60.2 & 67.6 \\
\rowcolor{Gray}
\textbf{+ Task Evidence (Ours)} & \cmark & \cmark & \textbf{NLL} & \textbf{61.5} & \textbf{58.7} & \textbf{66.1} & \textbf{65.2} & \textbf{61.1} & \textbf{70.3} \\
\bottomrule
\end{tabular}%
}
\end{table*}

\begin{table}[t]
\centering
\small
\caption{\textbf{Comparison of reward types for Stage II.}
Results are evaluated on Qwen2.5-VL-7B with $K=8$ and $K=32$.
Label reward provides medium-granularity but annotation-biased supervision.
RLVR-style correctness reward is task-level but sparse.
NLL provides dense task evidence and performs best across benchmarks.
\textbf{Bold} denotes the best performance.}
\label{tab:reward_type}
\setlength{\tabcolsep}{4.8pt}
\renewcommand{\arraystretch}{1.12}
\resizebox{\linewidth}{!}{%
\begin{tabular}{lllcccccc}
\toprule
\multirow{2}{*}{\textbf{Reward Type}} 
& \multicolumn{2}{c}{\textbf{Reward Properties}} 
& \multicolumn{3}{c}{\textbf{$K=8$}} 
& \multicolumn{3}{c}{\textbf{$K=32$}} \\
\cmidrule(lr){2-3} \cmidrule(lr){4-6} \cmidrule(lr){7-9}
& \textbf{Granularity} & \textbf{Bias}
& \textbf{V-MME} & \textbf{LVB} & \textbf{MLVU} 
& \textbf{V-MME} & \textbf{LVB} & \textbf{MLVU} \\
\midrule

\rowcolor{Gray}
\multicolumn{9}{l}{\textbf{Stage II reward signal}} \\

Label 
& Frame-level / medium 
& Annotation-biased 
& 58.3 & 57.6 & 64.9 
& 63.8 & 60.2 & 67.6 \\

RLVR-style correctness 
& Sample-level / sparse 
& Task-level 
& 57.3 & 56.5 & 65.3 
& 62.1 & 58.9 & 66.7 \\

\textbf{NLL} 
& \textbf{Token-level / dense} 
& \textbf{Task-level} 
& \textbf{61.5} & \textbf{58.7} & \textbf{66.1} 
& \textbf{65.2} & \textbf{61.1} & \textbf{70.3} \\

\bottomrule
\end{tabular}%
}
\end{table}

\subsection{Ablation Studies}
\label{sec:ablation}

\paragraph{Cross-environment Stage II Training.}
A natural concern is whether the posterior inference stage must be repeated for each downstream MLLM, since the NLL reward is obtained from a frozen MLLM environment. We compare two selectors trained with different Stage II environments and evaluate both on InternVL-3-8B. The first selector is trained with Qwen2.5-VL-7B as the frozen environment, which is our default setting. The second selector is trained by replacing the Stage II environment with InternVL-3-8B.

As shown in Table~\ref{tab:cross_env_training}, using InternVL-3-8B as the Stage II environment yields nearly identical performance to the selector trained with Qwen2.5-VL-7B. This result indicates that Stage II does not need to be re-run for each downstream MLLM in our setting. It suggests that the posterior update does not merely overfit to the idiosyncratic preference of one frozen environment but helps move the selector more useful across MLLMs.

\paragraph{Efficacy of Bayesian Posterior Update and Task Evidence.}
Table~\ref{tab:ablation_component} presents results of different training strategies on Qwen2.5-VL. The SigLIP2 baseline selects frames based on cosine similarity. Label Prior (Stage I) improves the selection quality by aligning with label saliency, yet the bias of static annotations limits the performance. Crucially, applying Reinforcement Learning with a label-based sparse reward does not align with the inference dynamics of the downstream model, and the reward signal is overly sparse, failing to produce meaningful gains over the Label Prior. In contrast, our method utilizing the Negative Log-Likelihood (NLL) as a dense reward signal achieves significant performance improvements across all metrics.

\paragraph{Impact of Frame Budget and Saturation.}
We further investigate the efficiency-performance trade-off by varying the number of selected frames $K$ from 8 to 128, as detailed in Table~\ref{tab:ablation_frames}. More experiments are provided in Appendix~\ref{app:ablation}.

\section{Conclusion}
\label{sec:conclusion}

We formulate frame selection as Bayesian inference, using Bayesian GRPO to refine annotation-derived priors with downstream task evidence. By leveraging the MLLM's negative log-likelihood as a dense reward, our lightweight selector improves video-QA accuracy with modest overhead and transfers across MLLMs. Future work will explore streaming video.

\clearpage
\section*{Acknowledgements}

This work was supported by the National Natural Science Foundation
of China (Grant No. 62471287).

\bibliographystyle{plainnat}
\bibliography{reference}






\clearpage 
\appendix

\section{Overview}

This appendix provides supplementary details to support the findings presented in the main paper. Algorithm is provided in Section~\ref{app:argo}. Implementation details, including hyperparameters, evaluation prompts, and benchmark descriptions, are provided in Section~\ref{app:exp_details}. We then describe the training dataset construction in Section~\ref{app:datasets}. Finally, we present extended ablation studies and sensitivity analyses on state-of-the-art MLLMs in Section~\ref{app:ablation}, followed by additional visualization cases in Section~\ref{app:case_study}.

\section{Algorithm}
\label{app:argo}
The complete training procedure, illustrating the transition from Prior Elicitation to Posterior Inference, is summarized in Algorithm~\ref{alg:bayes_grpo}.
\begin{algorithm}[t]
\caption{Bayesian Frame Selection via Prior-Guided GRPO}
\label{alg:bayes_grpo}
\small
\begin{algorithmic}[1]
\STATE {\bf Input:} Video dataset $\mathcal{D}$, label dataset $\mathcal{D}_{\rm label}$, frozen MLLM, budget $K$, group size $G$, KL coefficient $\beta$.
\STATE {\bf Output:} Frame-selection policy $\pi_\theta$.

\vspace{0.2em}
\STATE \textcolor{gray}{\textit{// Stage I: prior elicitation}}
\STATE Initialize policy parameters $\theta$.
\FOR{each batch $(\mathcal{V}, \mathcal{Q}, \mathbf{y}_{\rm pref}) \in \mathcal{D}_{\rm label}$}
    \STATE Compute frame scores $\mathbf{s}=f_\theta(\mathcal{V},\mathcal{Q})$.
    \STATE Update $\theta$ by minimizing the BTL ranking loss $\mathcal{L}_{\rm prior}$ in Eq.~\ref{eq:prior_loss}.
\ENDFOR
\STATE Freeze the elicited prior policy $\pi_{\rm prior} \leftarrow \pi_\theta$.

\vspace{0.2em}
\STATE \textcolor{gray}{\textit{// Stage II: posterior inference}}
\FOR{each batch $(\mathcal{V}, \mathcal{Q}, \mathbf{y}) \in \mathcal{D}$}
    \STATE Sample $G$ frame-selection trajectories $\{\mathbf{z}_i\}_{i=1}^G \sim \pi_\theta(\cdot \mid \mathbf{s})$.
    \FOR{$i=1,\ldots,G$}
        \STATE Evaluate task evidence $R(\mathbf{z}_i)=\log P_{\rm MLLM}(\mathbf{y}\mid \mathbf{z}_i,\mathcal{Q})$.
    \ENDFOR
    \STATE Normalize group-relative advantages $\{\hat{A}_i\}_{i=1}^G$.
    \STATE Compute prior regularization $D_{\rm KL}(\pi_\theta(\mathbf{z}_i)\,\|\,\pi_{\rm prior}(\mathbf{z}_i))$.
    \STATE Update $\theta$ by minimizing the Bayesian GRPO loss in Eq.~\ref{eq:grpo_loss}.
\ENDFOR
\end{algorithmic}
\end{algorithm}

\section{Implementation Details}
\label{app:exp_details}

In this section, we provide comprehensive implementation details to ensure the reproducibility of our method.


\subsection{Benchmarks.}
\textbf{Video-MME}~\cite{videomme} serves as a large-scale multimodal evaluation suite, assessing capabilities in video understanding, reasoning, and knowledge-grounded tasks. It incorporates both short and long-form videos to test context management strategies comprehensively. \textbf{LongVideoBench}~\cite{lvb} focuses specifically on long-form scenarios, with videos spanning tens of minutes to hours. It challenges models with tasks such as VQA, summarization, and temporal grounding, rigorously testing temporal reasoning efficiency. \textbf{MLVU}~\cite{mlvu} evaluates hierarchical comprehension. It tests models on multiple levels of reasoning, including frame-level recognition, clip-level temporal understanding, and video-level holistic comprehension, ensuring that the selection method preserves both local details and global context.

\subsection{Implementation Details.}
For all frame-selection methods evaluated under our controlled setting, candidate frames are uniformly sampled from the raw video at 1 FPS before selection. Given a video of duration $T$ seconds, this produces $N=T$ candidate frames, and each method selects $K$ frames from the same candidate pool. Unless otherwise specified, lightweight baselines including AKS, BOLT, Q-Frame, F2C, and our method use the same SigLIP2-SO400M features, ensuring that performance differences are not caused by different feature extractors.The lightweight selector consists of $L=4$ blocks of our proposed Video-Text Fusion Transformers. The hidden dimension is set to $D=1152$ to align with the visual and textual features. We train the selector on the GenS-Video-150K dataset with 150k video-question pairs. The training is conducted in two stages. In Stage I, we optimize the pairwise ranking loss with a learning rate of $1e^{-4}$. In Stage II, we fine-tune the selector using Bayesian GRPO with a learning rate of $5e^{-6}$. The Stage II environment is Qwen2.5-VL-7B, which is kept frozen throughout training. After training, the selector is directly deployed to other downstream MLLMs without additional posterior training. All training and evaluation experiments are conducted on 8 NVIDIA H100 (80GB) GPUs.

\subsection{Hyperparameters of Bayesian GRPO}
We detail the hyperparameters used in both label prior elicitation (Stage I) and task evidence inference (Stage II) in Table~\ref{tab:hyperparams}. The training was conducted on 8 NVIDIA H100 GPUs.

\begin{table}[h]
    \centering
    \caption{Hyperparameter settings for the two-stage training process.}
    \label{tab:hyperparams}
    \begin{tabular}{lcc}
    \toprule
    \textbf{Hyperparameter} & \textbf{Stage I: Label Prior} & \textbf{Stage II: Task Evidence} \\
    \midrule
    Batch Size & 32 & 1 \\
    Learning Rate & $1e^{-4}$ & $5e^{-6}$ \\
    Optimizer & AdamW & AdamW \\
    Weight Decay & $0.01$ & $0.01$ \\
    Epochs & 10 & 1 \\
    KL Coefficient ($\beta$) & - & 0.2 \\
    Group Size ($G$) & - & 8 \\
    PPO Epochs & - & 3 \\
    Clip Ratio ($\epsilon$) & - & 0.2 \\
    \bottomrule
    \end{tabular}
\end{table}

\subsection{Efficiency and Training Cost Analysis}
\label{sec:efficiency}

A central goal of our method is to preserve the inference efficiency of lightweight selectors while avoiding the heavy test-time overhead of MLLM-based scorers. We therefore report both computational cost and measured latency under a common setting. Following the evaluation protocol, candidate frames are sampled at 1 FPS. For a 300-second video with $K=32$, we estimate the selector-side computational cost with fp16 model weights and report the peak memory usage.

\begin{table}[t]
\centering
\caption{\textbf{Selector-side efficiency comparison.}
GFLOPs and peak memory are measured on a 300-second video sampled at 1 FPS with $K=32$.
BOLT, F2C, and ours use the same SigLIP2-SO400M feature extractor.}
\label{tab:efficiency}
\small
\setlength{\tabcolsep}{6pt}
\renewcommand{\arraystretch}{1.12}
\begin{tabular}{llcc}
\toprule
\textbf{Method} & \textbf{Encoder / Scorer} & \textbf{GFLOPs} & \textbf{Mem.} \\
 & & & \textbf{(GB)} \\
\midrule
\multicolumn{4}{l}{\textit{Lightweight feature-based scoring}} \\
\midrule
BOLT & SigLIP2-SO400M & 240 & 0.80 \\
F2C  & SigLIP2-SO400M & 241 & 0.80 \\
\textbf{Ours} & SigLIP2-SO400M + selector & \textbf{264} & \textbf{0.93} \\
\midrule
\multicolumn{4}{l}{\textit{MLLM-based scoring}} \\
\midrule
A.I.R. & CLIP-L + Qwen2.5-VL-7B & 72K--431K & 15.81 \\
GenS   & Aria-25B-A3.9B          & 315K      & 50.60 \\
\bottomrule
\end{tabular}
\end{table}

As shown in Table~\ref{tab:efficiency}, our selector introduces only a small overhead over SigLIP2-based lightweight methods, increasing the estimated cost from 240--241 GFLOPs to 264 GFLOPs and memory from 0.80GB to 0.93GB. In contrast, MLLM-based scorers require repeated large-model inference over candidate frames or windows, leading to orders-of-magnitude higher cost.

We further measure wall-clock latency on Video-MME using a single NVIDIA H100 GPU, with $K=32$ and candidate frames sampled at 1 FPS. The average per-sample latency is reported in Table~\ref{tab:latency}. Our method remains close to lightweight training-free methods and is substantially faster than generative MLLM-based scoring.

\begin{table}[t]
\centering
\small
\caption{\textbf{Average per-sample inference latency on Video-MME.} Latency is measured on a single NVIDIA H100 GPU with $K=32$ and 1 FPS candidate sampling.}
\label{tab:latency}
\begin{tabular}{lc}
\toprule
\textbf{Method} & \textbf{Latency (s/sample)} \\
\midrule
BOLT & 1.29 \\
Ours & 1.38 \\
Q-Frame & 1.53 \\
GenS & 214.6 \\
\bottomrule
\end{tabular}
\end{table}

Although our method is not training-free, the training cost is incurred only once. Stage I learns the label prior from keyframe annotations, and Stage II uses a frozen Qwen2.5-VL-7B environment for posterior refinement. After training, the selector is reused as a standalone module and directly transferred to different downstream MLLMs without re-running Stage II. Therefore, the training cost is amortized across benchmarks, frame budgets, and downstream MLLMs, while inference requires only one selector forward pass for each video-query pair.

\subsection{Evaluation Details}

 We utilize the following structured prompt to evaluate MLLMs across all video benchmarks. The temperature is set to 0.0 to ensure deterministic generation. Inference is performed using \texttt{transformers==4.57}, and accuracy is calculated using Exact Match.

\begin{center}
\setlength{\fboxsep}{8pt} 
\fcolorbox{gray!50}{gray!10}{
    \begin{minipage}{0.95\linewidth} 
        \textbf{Evaluation Prompt Template}\\[0.5em] 
        \ttfamily\small 
        \textless Selected Frames\textgreater \\
        You are a helpful assistant. \\
        \textless Questions\textgreater \\
        \textless Options\textgreater \\
        Answer with the option's letter from the given choices directly.
    \end{minipage}%
}
\end{center}

\section{Dataset Details}
\label{app:datasets}

We utilize the GenS-Video-150K~\cite{gens} dataset for training. The dataset annotation pipeline leverages GPT-4o to construct grounded video QAs and score frame relevance from 0 (irrelevant) to 5 (most relevant). To ensure comprehensive coverage of video understanding capabilities, the dataset includes 12 specific question types categorized into three primary domains: (1) \textbf{Reasoning Tasks}, which test the model's ability to infer logical relationships involving objects, actions, spatial configurations, and temporal changes; (2) \textbf{Perception Tasks}, focusing on the basic visual understanding of scenes, attributes, and spatial layouts; and (3) \textbf{Specialized Tasks}, which assess fine-grained capabilities including video detail referring, counting, OCR, and temporal sequence perception.

\section{More Experimental Results}
\label{app:ablation}

In this section, we provide additional ablation studies and sensitivity analyses on \textbf{Qwen3-VL-8B} and \textbf{InternVL-3.5-8B}. These experiments further validate the effectiveness and generalizability of our proposed Bayesian GRPO framework across different model architectures.

\subsection{Ablation on Training Components on SOTA Models}

We investigate the contribution of each training stage on Qwen3-VL and InternVL-3.5. Table~\ref{tab:ablation_extended1} and Table~\ref{tab:ablation_extended2} demonstrates that while the Label Prior (Stage I) provides a reasonable initialization, the Bayesian GRPO update (Stage II) with NLL-based Task Evidence is essential for maximizing performance.

\begin{table*}[h]
\centering
\caption{\textbf{Extended Ablation Study on Qwen3-VL-8B.} Comparison of training strategies across different frame budgets ($K=8, 32$).}
\label{tab:ablation_extended1}
\resizebox{\textwidth}{!}{%
\begin{tabular}{l ccc | ccc ccc}
\toprule
\multirow{2}{*}{\textbf{Method}} & \multicolumn{3}{c}{\textbf{Training Components}} & \multicolumn{3}{c}{\textbf{K = 8}} & \multicolumn{3}{c}{\textbf{K = 32}} \\
\cmidrule(lr){2-4} \cmidrule(lr){5-7} \cmidrule(lr){8-10}
 & \textbf{Stage I} & \textbf{Stage II} & \textbf{Reward} & \textbf{V-MME} & \textbf{LVB} & \textbf{MLVU} & \textbf{V-MME} & \textbf{LVB} & \textbf{MLVU} \\
\midrule
SigLIP2 & \xmark & \xmark & - & 55.4 & 56.7 & 61.0 & 62.1 & 59.8 & 69.2 \\
Label Prior & \cmark & \xmark & - & 57.3 & 58.1 & 64.3 & 63.9 & 61.4 & 71.0 \\
+ Label Reward & \cmark & \cmark & Label & 57.1 & 57.4 & 64.9 & 64.1 & 61.0 & 71.4 \\
\rowcolor{Gray}
\textbf{+ Task Evidence} & \cmark & \cmark & \textbf{NLL} & \textbf{58.3} & \textbf{58.2} & \textbf{66.7} & \textbf{65.3} & \textbf{62.7} & \textbf{72.7} \\
\bottomrule
\end{tabular}%
}
\end{table*}

\begin{table*}[h]
\centering
\caption{\textbf{Extended Ablation Study on InternVL-3.5-8B.} Comparison of training strategies across different frame budgets ($K=8, 32$).}
\label{tab:ablation_extended2}
\resizebox{\textwidth}{!}{%
\begin{tabular}{l ccc | ccc ccc}
\toprule
\multirow{2}{*}{\textbf{Method}} & \multicolumn{3}{c}{\textbf{Training Components}} & \multicolumn{3}{c}{\textbf{K = 8}} & \multicolumn{3}{c}{\textbf{K = 32}} \\
\cmidrule(lr){2-4} \cmidrule(lr){5-7} \cmidrule(lr){8-10}
 & \textbf{Stage I} & \textbf{Stage II} & \textbf{Reward} & \textbf{V-MME} & \textbf{LVB} & \textbf{MLVU} & \textbf{V-MME} & \textbf{LVB} & \textbf{MLVU} \\
\midrule
SigLIP2 & \xmark & \xmark & - & 56.7 & 58.1 & 64.8 & 60.9 & 60.6 & 69.7 \\
Label Prior & \cmark & \xmark & - & 57.6 & 58.9 & 66.2 & 62.4 & 62.3 & 70.7 \\
+ Label Reward & \cmark & \cmark & Label & 57.3 & 58.1 & 66.5 & 62.8 & 62.7 & 70.3 \\
\rowcolor{Gray}
\textbf{+ Task Evidence} & \cmark & \cmark & \textbf{NLL} & \textbf{58.3} & \textbf{59.6} & \textbf{67.6} & \textbf{64.1} & \textbf{63.4} & \textbf{71.4} \\
\bottomrule
\end{tabular}%
}
\end{table*}

\subsection{Sensitivity Analysis: Frame Budget on SOTA Models}

We analyze the impact of the frame budget $K$ on Qwen2.5-VL, Qwen3-VL and InternVL-3.5, both models exhibit a consistent trend: performance improves rapidly as $K$ increases from 8 to 32 and saturates around $K=64$. This confirms that our selector effectively identifies the most critical evidentiary frames within a constrained budget, regardless of the downstream MLLM architecture.

\begin{table}[t]
\centering
\caption{Performance with varying frame budgets ($K$) on Video-MME, LongVideoBench, and MLVU. All results are evaluated using \textbf{Qwen2.5-VL}.}
\label{tab:ablation_frames}
\begin{tabular}{lccccc}
\toprule
\textbf{Metric} & \textbf{$K=8$} & \textbf{$K=16$} & \textbf{$K=32$} & \textbf{$K=64$} & \textbf{$K=128$} \\
\midrule
Video-MME & 61.5 & 63.7 & 65.2 & 66.2 & 66.9 \\
LVB & 58.2 & 59.8 & 61.1 & 62.4 & 63.3 \\
MLVU (Avg) & 66.1 & 68.7 & 70.3 & 73.3 & 73.4 \\
\bottomrule
\end{tabular}
\end{table}

\begin{figure*}[h]
    \centering
    \begin{minipage}[t]{0.48\textwidth}
        \centering
        \captionof{table}{Performance with varying frame budgets ($K$) on \textbf{Qwen3-VL-8B}.}
        \label{tab:ablation_qwen3}

        \renewcommand{\arraystretch}{1.2} 
        \begin{tabular*}{\linewidth}{@{\extracolsep{\fill}}lccccc}
            \toprule
            \textbf{Metric} & \textbf{8} & \textbf{16} & \textbf{32} & \textbf{64} & \textbf{128} \\ 
            \midrule
            Video-MME & 58.3 & 59.7 & 65.3 & 65.6 & 67.2 \\ 
            LVB       & 58.2 & 59.2 & 62.7 & 63.7 & 63.9 \\ 
            MLVU      & 66.7 & 70.0 & 72.7 & 76.3 & 77.7 \\ 
            \bottomrule
        \end{tabular*}

        \includegraphics[width=\linewidth]{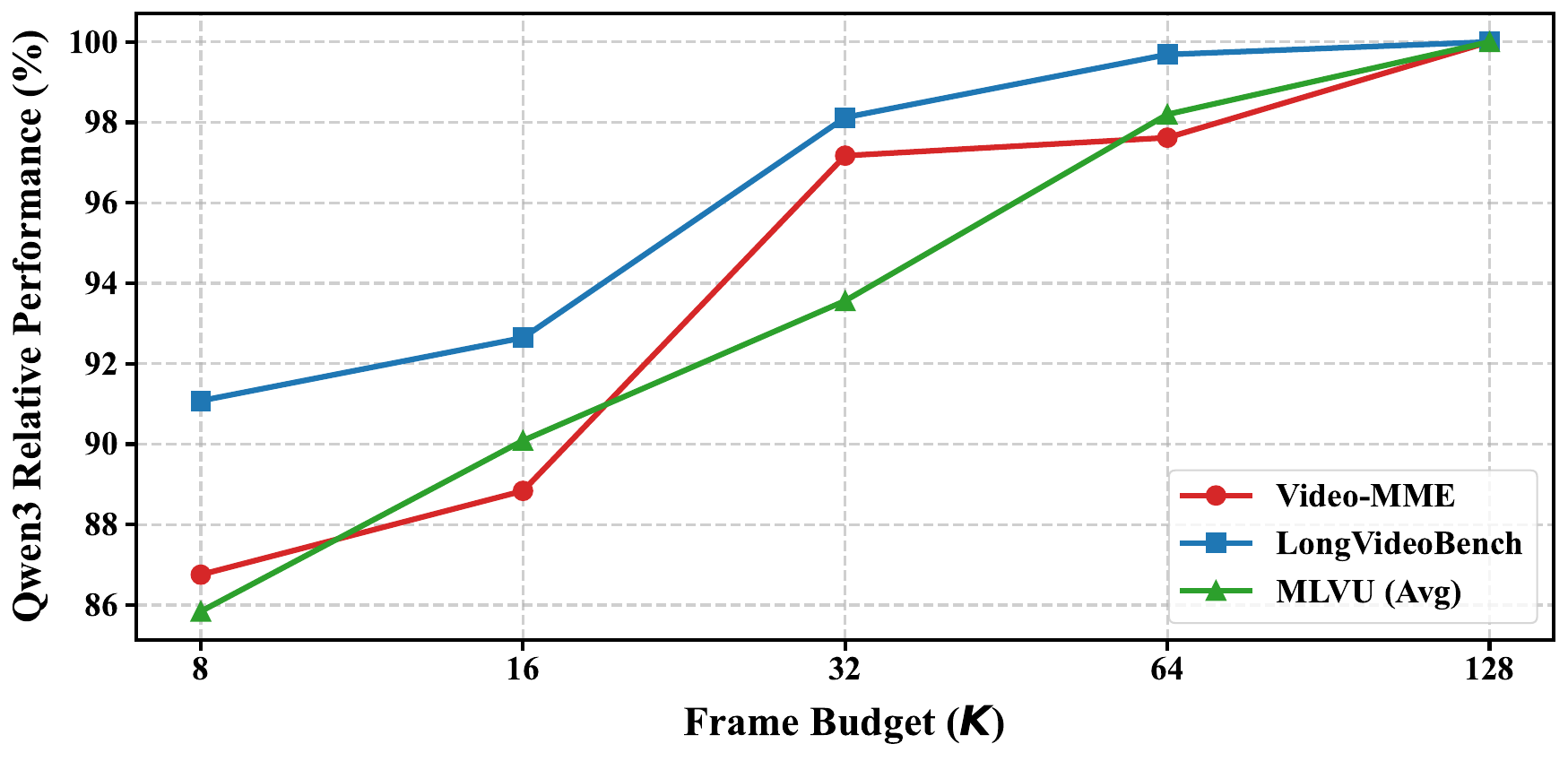} 
        \captionof{figure}{Relative performance trend on \textbf{Qwen3-VL-8B}.}
    \end{minipage}
    \hfill
    \begin{minipage}[t]{0.48\textwidth}
        \centering
        \captionof{table}{Performance with varying frame budgets ($K$) on \textbf{InternVL3.5-8B}.}
        \label{tab:ablation_internvl35}

        \renewcommand{\arraystretch}{1.2} 
        \begin{tabular*}{\linewidth}{@{\extracolsep{\fill}}lccccc}
            \toprule
            \textbf{Metric} & \textbf{8} & \textbf{16} & \textbf{32} & \textbf{64} & \textbf{128} \\ 
            \midrule
            Video-MME & 58.3 & 59.6 & 64.1 & 63.7 & 64.5 \\ 
            LVB       & 59.6 & 60.7 & 63.4 & 62.7 & 63.3 \\ 
            MLVU      & 67.6 & 69.8 & 71.4 & 73.1 & 72.3 \\ 
            \bottomrule
        \end{tabular*}
        \includegraphics[width=\linewidth]{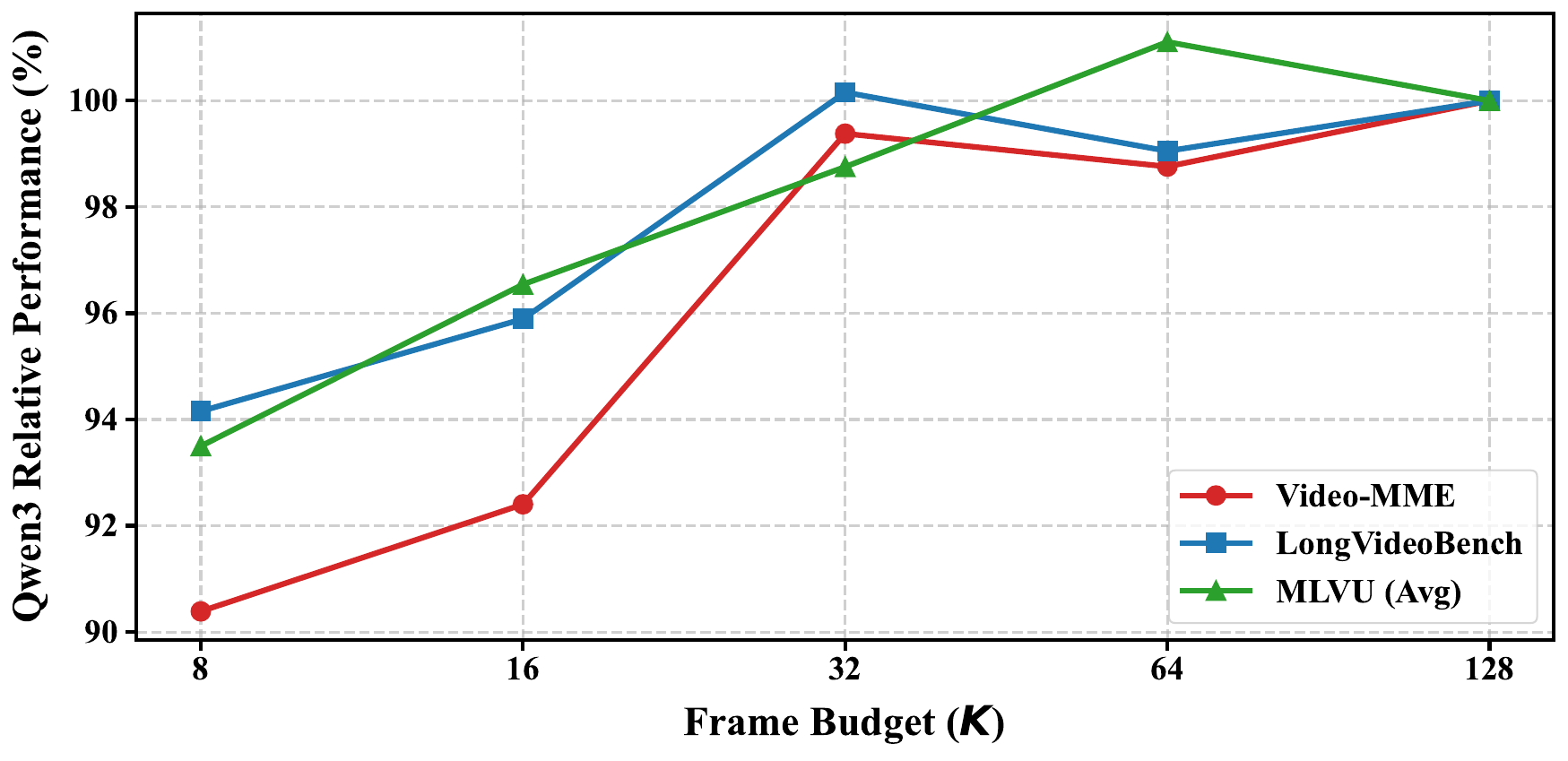}
        \captionof{figure}{Relative performance trend on \textbf{InternVL-3.5-8B}.}
    \end{minipage}
    \label{fig:ablation_sota}
\end{figure*}

\begin{figure*} 
    \centering
    \begin{subfigure}{0.95\textwidth}
        \centering
        \includegraphics[width=\linewidth]{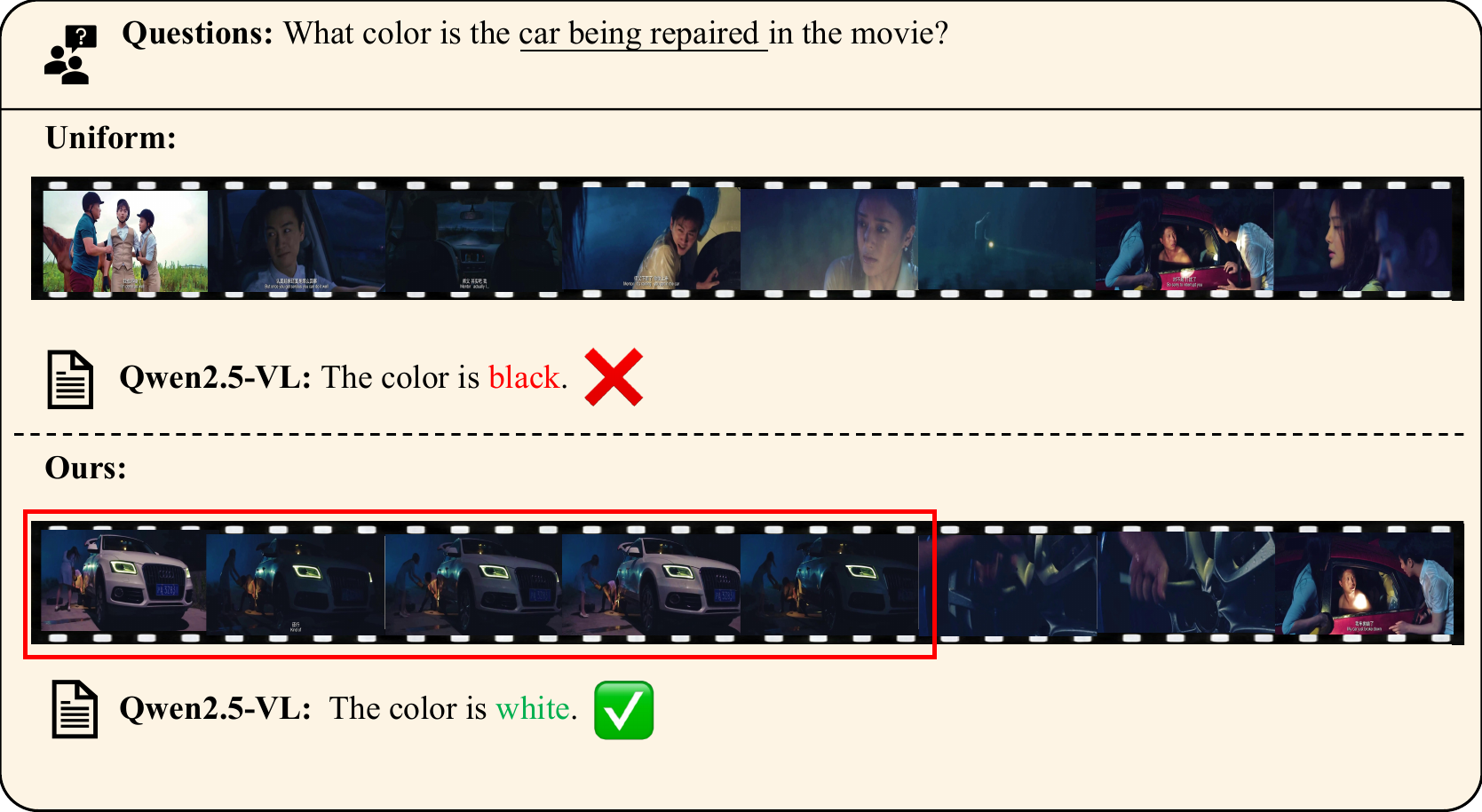}
        \label{fig:case1}
    \end{subfigure}

    \caption{\textbf{Qualitative Examples 1 of Frame Selection.} From top to bottom, we visualize the frames selected by our method compared to the baseline.}
    \label{fig:case_study_full2}
\end{figure*}

\section{More Visualization Results}
\label{app:case_study}

We provide additional visualization cases in Figure~\ref{fig:case_study_full1}, Figure~\ref{fig:case_study_full2}. These examples further illustrate how our Bayesian GRPO selector effectively identifies evidentiary frames that align with the reasoning requirements of the MLLM, contrasting with the static label priors.

\begin{figure*} 
    \centering
    \begin{subfigure}{0.95\textwidth}
        \centering
        \includegraphics[width=\linewidth]{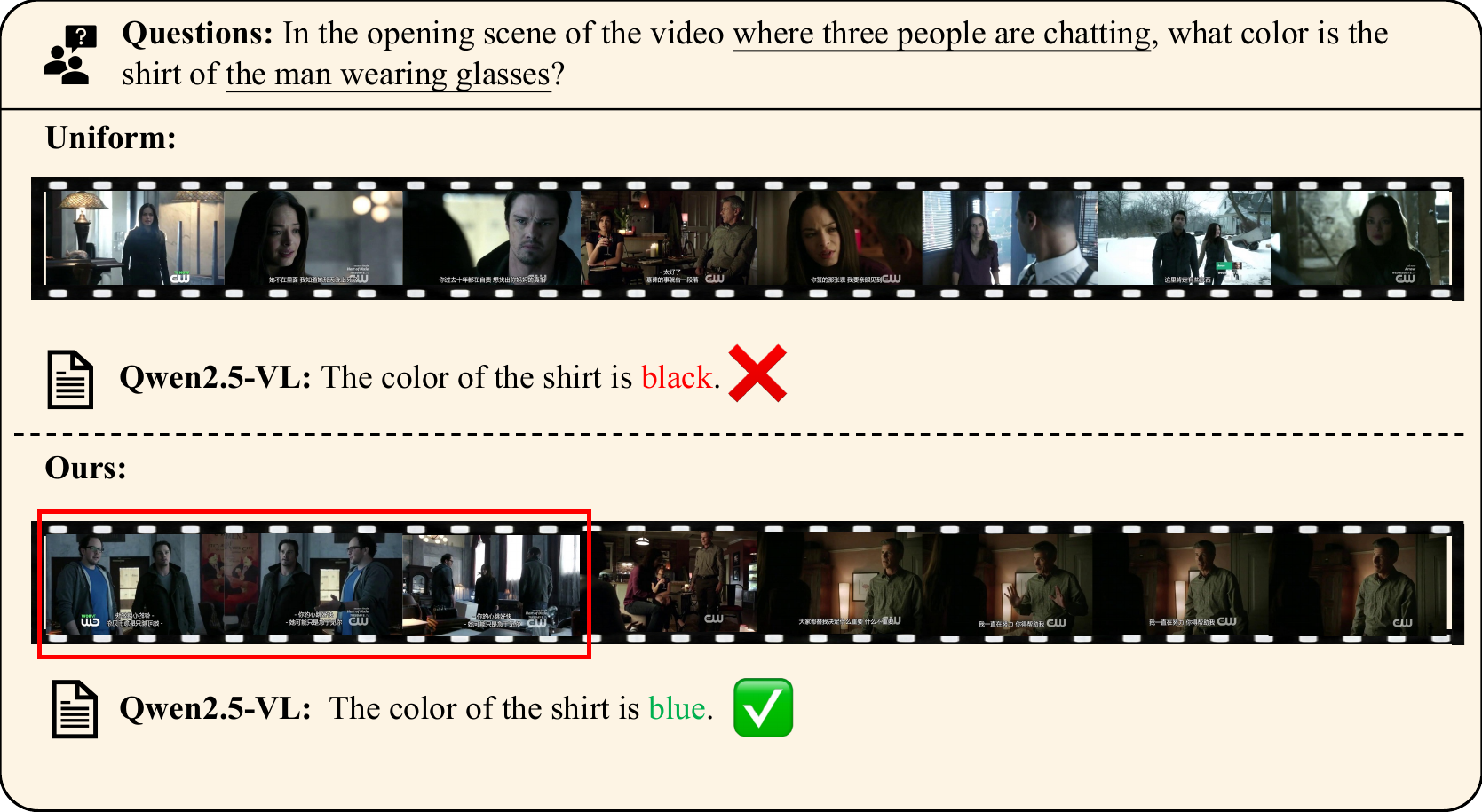}
        \label{fig:case1}
    \end{subfigure}

    \begin{subfigure}{0.95\textwidth}
        \centering
        \includegraphics[width=\linewidth]{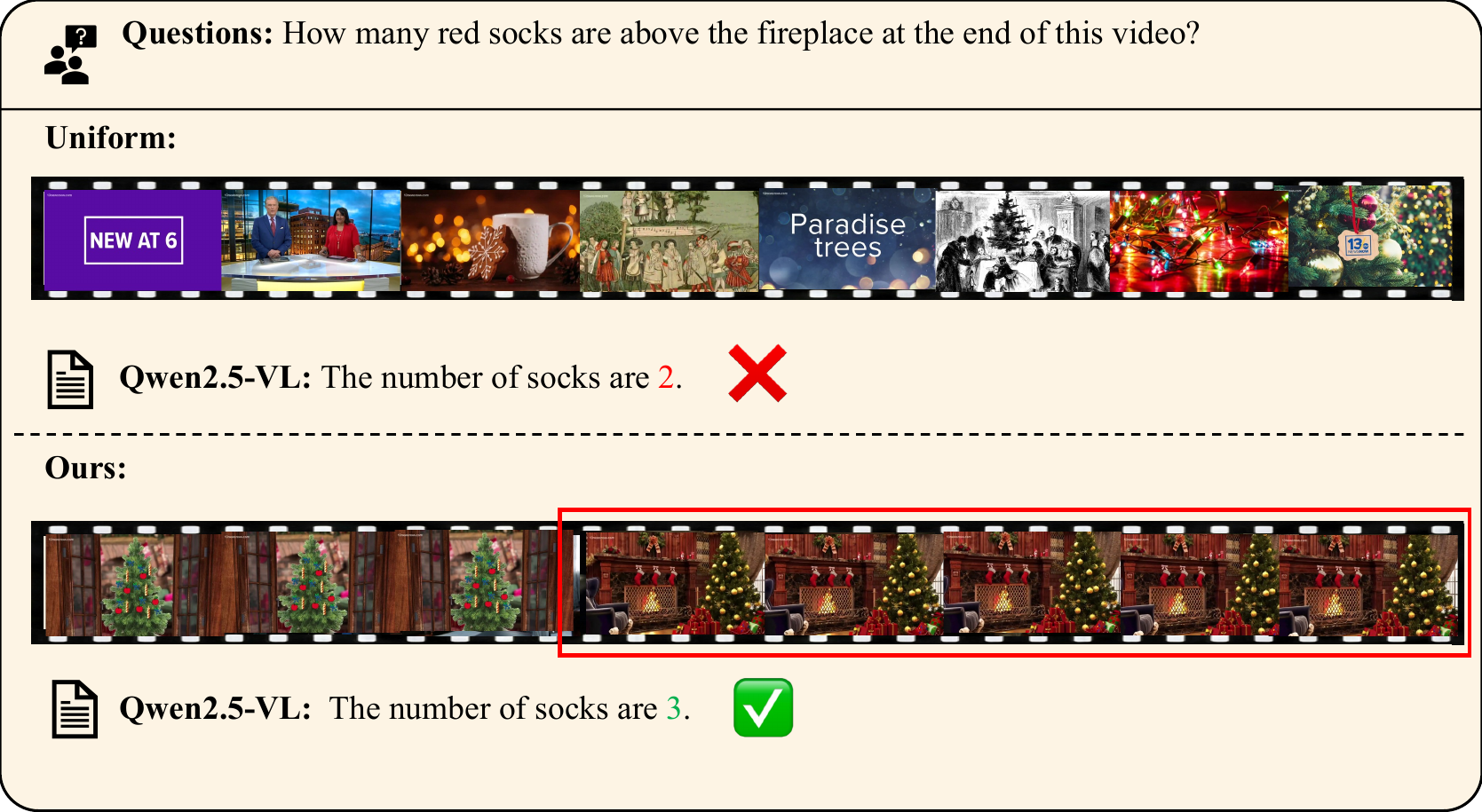}
        \label{fig:case2}
    \end{subfigure}

    \caption{\textbf{Qualitative Examples 2 and 3 of Frame Selection.} From top to bottom, we visualize the frames selected by our method compared to the baseline.}
    \label{fig:case_study_full1}
\end{figure*}


\clearpage

\end{document}